\documentclass[preprint,12pt]{elsarticle}
\usepackage{rotating}
\usepackage{enumerate}
\usepackage{amsmath}
\usepackage{lineno}
\usepackage{graphicx}
\usepackage{multirow}
\usepackage{amssymb}
\usepackage{float}
\usepackage{amsfonts}
\usepackage{algorithmicx} 
\usepackage{algorithm} 
\usepackage{algpseudocode} 
\usepackage{amsthm}

\usepackage{mathtools} 
\usepackage{subfig}
\usepackage{ulem}
\usepackage{bm}

\newtheorem{myDef}{Definition} 
\newtheorem{theorem}{Theorem}

\begin{document}

\begin{frontmatter}

\title{Hierarchical clustering by aggregating representatives in sub-minimum-spanning-trees}

\author[a,b]{Wen-Bo Xie}
\author[a,b]{Zhen Liu\corref{cor1}}
\author[c]{Jaideep Srivastava}

\address[a]{Big Data Research Center, School of Computer Science and Engineering, University of Electronic Science and Technology of China, Chengdu 611731, People's Republic of China.}
\address[b]{Web Sciences Center, School of Computer Science and Engineering, University of Electronic Science and Technology of China, Chengdu 611731, People's Republic of China.}
\address[c]{College of Science and Engineering, University of Minnesota, Minneapolis MN 55455, United States of America.}

\cortext[cor1]{Corresponding author at: Big Data Research Center, University of Electronic Science and Technology of China, Chengdu 611731, China. E-mail: quake.liu0625@gmail.com, quake@uestc.edu.cn (Zhen Liu)}



\begin{abstract}

One of the main challenges for hierarchical clustering is how to appropriately identify the representative points in the lower level of the cluster tree, which are going to be utilized as the roots in the higher level of the cluster tree for further aggregation. However, conventional hierarchical clustering approaches have adopted some simple tricks to select the “representative” points which might not be as representative as enough. Thus, the constructed cluster tree is less attractive in terms of its poor robustness and weak reliability. Aiming at this issue, we propose a novel hierarchical clustering algorithm, in which, while building the clustering dendrogram, we can effectively detect the representative point based on scoring the reciprocal nearest data points in each sub-minimum-spanning-tree. Extensive experiments on UCI datasets show that the proposed algorithm is more accurate than other benchmarks. Meanwhile, under our analysis, the proposed algorithm has $O(n\log n)$ time-complexity and $O(\log n)$ space-complexity, indicating that it has the scalability in handling massive data with less time and storage consumptions.

\end{abstract}

\begin{keyword}
Hierarchical Clustering; Reciprocal Nearest Neighbor; Roots Detection; Boundary-based Score
\end{keyword}

\end{frontmatter}

\section{Introduction}

With the explosive growth of information, some latent patterns behind the mass data without class annotations are really hard for people to understand. Thus, unsupervised learning algorithms, especially clustering algorithms, are widely applied to gain insight into complex datum and to discover subversive knowledge. The clustering algorithm has been followed more than one hundred years, and several types of clustering algorithms have been proposed, e.g., the partition algorithms \cite{Jain 2010, Shah 2017}, the density-based algorithms \cite{Rodrigue 2014, Wang 2009}, the affinity propagation algorithms \cite{Frey 2007}, the spectral-based algorithms \cite{Filippone 2008,Huang 2019}, and the hierarchical clustering algorithms \cite{Kobren 2017}. Hierarchical clustering algorithms organize data into tree-structure that provides people with intuitive informational contents. Benefiting from the structured and highly explicable results, hierarchical clustering algorithms are widely applied in the tasks of scientific analysis in disparate fields, such as gene-related analysis in biomedical science \cite{Song 2016, Ma 2017, Maganga 2020}, community detection in social networks \cite{Newman 2004, Wilkinson 2004, Rattigan 2007, Fortunato 2010}, pattern discovery in environmental assessment \cite{Dugan 2017}, and so on. 

Most hierarchical clustering algorithms link data points to generate a tree-structured result, say cluster tree \cite{Reddy 2018}. Some classical algorithms merge sub-clusters by using different linkage methods such as average-linkage-based algorithm \cite{Ward Jr 1963}, single-linkage-based algorithm \cite{Sneath 1973}, and so on. Up to now, the above classical hierarchical clustering algorithms are still widely used owing to their intuitive definition and simple operation, while they may fail in some cases because of producing elongated clusters, and are usually not able to scale up well \cite{Rokach 2009}.  

Therefore, on one hand, researchers proposed some novel strategies to speed up classical algorithms, e.g., the sampling-based methods and the stepwise algorithms. ROCK \cite{Guha 2000} and CURE\cite{Guha 2001} are typical sampling-based methods, which calculate the distance between clusters based on the fixed number of representative data points of clusters. CHAMELEON \cite{Karypis 1999}, one of the classical stepwise algorithms, applies nearest neighbor graph to divide data nodes into small clusters to reduce the number of iterations. Although the sampling-based algorithms and the stepwise algorithms are increasingly used in clustering applications due to their efficiency and extreme effectiveness in capturing arbitrarily shaped clusters \cite{Reddy 2018}, the issue of parameter sensitivity existing among these methods undermines the accuracy of clustering results. On the other hand, some researchers proposed new hierarchical clustering frameworks, e.g., clustering feature tree (CF-Tree) based algorithms, like BIRCH \cite{Zhang 1997} and its variants \cite{Kobren 2017,Ryu 2020}, realize the clustering operation by constructing binary trees for stream data.  However, CF-tree-based algorithms are highly sensitive to the different orders of data input.

In this paper, we propose a novel hierarchical clustering algorithm based on scoring the reciprocal-nearest-neighbors (RNNs for short), in which the RNNs points with the higher topological importance are treated as the representative data points (also named roots) of clusters to guide the aggregation of data points hierarchically. We specially convert the data points into a graph and detect the root in a given cluster by introducing four scoring functions and two of their combined forms. Moreover, to deal with the failure of the scoring method in some cases, we also propose a novel strategy to sample pairwise data points on the data boundary, which is used to measure the closeness between roots and data boundary. Extensive experiments on University of California Irvine (UCI) data sets \cite{Lichman 2013} show that the proposed algorithm performs overall better than classical benchmarks (i.e., group average method\cite{Ward Jr 1963}, CHAMELEON \cite{Karypis 1999}, CURE \cite{Guha 2001}) and the state-of-the-art method PERCH \cite{Kobren 2017} and RSC \cite{Xie 2020}.

The rest of the article is organized as follows. In Section 2, we review the related hierarchical clustering algorithms, i.e., classical linkage-based algorithms, sampling-based algorithms, stepwise algorithms, and CF-Tree-based algorithms. In Section 3, we introduce the formal description of clustering problem and the definitions of cluster tree and reciprocal nearest neighbors (RNNs for short). Then, in Section 4, we present a novel hierarchical clustering algorithm based on detecting sub-clusters' roots by scoring the RNNs data points. Meanwhile, some theorems and their proofs are presented to ensure the validity of the proposed scoring measures. We also perform experiments to analyze the scoring measures and compare the new algorithm with benchmark algorithms in Section 5. Furthermore, we analyze both time and space complexity of the proposed algorithm in Section 6. Finally, we summarize our work and give some future directions of this study in the last section.
%
%
\section{Related Works}

As a big family, various clustering algorithms have their own specific strategies. Partition algorithms, e.g., k-means \cite{Jain 2010} and robust continuous clustering \cite{Shah 2017}, assign data points to the nearest cluster center which will be updated after each iteration. Density-based algorithms, e.g., DBSCAN \cite{Birant 2007} and density-and-distance-based clustering \cite{Rodrigue 2014}, calculate the local density and assign points to the nearest density center. The affinity propagation algorithm \cite{Frey 2007} is well-known for its nonconformist underlying mechanism that recursively transmits two types of messages, responsibility and availability, between nodes until suitable exemplars and corresponding clusters emerge. The above algorithms represent clustering results by using a map from nodes to cluster identities. Meanwhile, some structured results that can encode more fine-grained information, e.g., cluster tree, are more favored in some specified applications. 

The subsets of a cluster tree also have tree-structure that usually correspond to internal nodes. Thus classical hierarchical aggregation clustering algorithms (HAC for short) were proposed based on different linkage methods to calculate the distance between internal nodes and merge them into new subsets. The group average clustering method (HAC-A) \cite{Ward Jr 1963} merges sub-clusters according to the average distance of the pairwise data points from two different clusters, i.e., average-linkage. Differently, the Sneath-Sokal method (HAC-S) \cite{Sneath 1973} considers that the distance between two clusters can be measured by calculating the distance between the two nearest data points from the two clusters respectively, i.e., single-linkage. In contrast, the furthest neighbor method (HAC-C) \cite{King 1967} calculates the distance between two clusters based on the longest distance existing among two members from the two clusters respectively, i.e., complete-linkage. Considering the efficiency, some equivalents \cite{Murtagh 1983, Lopez-Sastre 2012} were proposed to reduce the time cost by optimizing the update process of NN-chain. The above algorithms, however, do not essentially change the calculation rules of the distance between clusters, which are still lack of scalability \cite{Rokach 2009} and may perform poorly for being apt to produce elongated clusters.

To improve the robustness and scalability of HAC algorithms, some better strategies were proposed. The sampling-based methods (e.g., ROCK \cite{Guha 2001} and CURE \cite{Guha 2000}) consider the distance between clusters on the basis of the representatives that are obtained by shrinking well-selected data points to the center of the clusters. The stepwise algorithms accelerate the clustering process by using a na\"ive process to construct small clusters in advance. CHAMELEON \cite{Karypis 1999} is a typical stepwise algorithm that applies k-nearest-neighbor graph to partition data points before hierarchical aggregation, yet the parameter $k$ is highly sensitive to the model. Similarly, Bouguettaya et al. \cite{Bouguettaya 2015} combine k-means with classical agglomerative hierarchical clustering algorithm, e.g. group average method and Sneath-Sokal method, to perform efficient clustering, in which k-means is applied to produce middle-level clusters and a selected agglomerative hierarchical clustering algorithm is applied to build the final cluster tree. On the other hand, some researchers proposed new frameworks to address the issues of robustness and scalability, simultaneously. The clustering feature tree (CF-Tree) based algorithm BIRCH \cite{Zhang 1997} realizes the clustering operation by constructing binary trees with stream data. BIRCH can obtain a desired clustering results by just one round of scanning over the dataset and a few additional operations for the optimization of the cluster tree. However, BIRCH is fragile on the datasets with outliers and sensitive to the sequences of data. To be more robust, PERCH \cite{Kobren 2017} employs bounded box approximations, masking-based rotation and balance-based rotation operations in the tree-building process to prevent the negative impacts of the outliers and ensure to generate a shallow tree. 

\section{Term definition}

Before giving the detailed description of the proposed algorithm, we first introduce the concept of clustering problem.

Given a dataset $X=\{x_i|i=1,2,\dots,n\}, x_i \in \mathbb{R}^m$, the clustering problem is to divide the dataset $X$ into a set of disjoint subsets $\mathcal{C}$, say clusters, of which the union covers the dataset. A high-quality clustering is that the data points in any particular subset have closer relationships to each other than those in other subsets. Commonly, clustering algorithms, such as partition algorithms and density-based algorithms, represent clustering results by using a map from data points to clusters, $\mathcal{C}: X \to \{C_i|i=1,2,\cdots,k\}$. In this paper, we mainly focus on the issues of hierarchical clustering, an important branch of clustering algorithms. The output of hierarchical clustering is a cluster tree which can be defined as follows.
\begin{myDef}
Cluster Tree \cite{Krishnamurthy 2012}. A cluster tree $\tau$ on a dataset $\{x_i|i=1,2,\dots,n\}$ is a collection of subsets such that $C_0 \triangleq \{x_i|i=1,2,\dots,n\} \in \tau$ and for any subsets $C_i, C_j \in \mathcal{C}$. either $C_i \subset C_j$ or $C_j \subset C_i$ or $C_i \cap C_j = \varnothing$. In particular, there must exist a set of disjoint clusters $\{C_i|i=1,2,\dots,k\}$ such that $\bigcup_{i=1}^k C_i = \mathcal{C}$.
\end{myDef}
Cutting the tree at a given height will give a partitioning clustering at a selected precision. Next, we introduce another important concept, reciprocal nearest neighbors (RNNs), which is going to be frequently used in the later part of this paper. The RNNs is defined as,

\begin{myDef}
    RNNs \cite{Murtagh 1983}. If there exists two data points denoted by $x$ and $y$, the nearest neighbors for $x$ and $y$ could be $\delta_x$ and $\delta_y$, respectively. When two conditions, say $\delta_x = y$ and $\delta_y = x$, hold simultaneously, we call $(x, y)$ a pair of RNNs.
\end{myDef}

In this study, we treat a pair of RNNs as the core of the cluster as it could play a key role in detecting clusters accurately.

\section{Algorithm}

There are two main operations contained in the proposed algorithm: (1) construction of sub-MSTs, which is to merge data points into small groups (sub-clusters); (2) detection of roots, which is to detect the representative data points in each sub-clusters. Besides, an additional operation is also proposed to handle anomalous RNNs data points.

\subsection{Construction of sub-MSTs}

The primary purpose of this step is to group data points into different small clusters, each of which has a tree-style structure, and the main idea inherits from our previous work \cite{Xie 2020}. In order to fulfil this task, we treat the data points as nodes, and apply a searching process to construct a graph that could be considered as a fragmentized minimum-spanning-tree, in which each detected subgraph is a fragment of the minimum-spanning-tree, say sub-MST. 

Given a dataset with $n$ data points $X=\{x_i\}^n, x_i \in \mathbb{R}^m $, and the distances between them are defined as $\mathfrak{d}_{ij}=dist(x_i,x_j)$. Note that, in order to maintain geometric characteristics, we only adopt Euclidean distance in this article. Firstly, we initiate a candidate set $\Gamma$ containing all $n$ nodes, and start the searching process from a randomly selected node $x_i$ therein. Then $x_i$ will connect to its nearest neighbor $\delta_i^{(1)}$, where the nearest neighbor of the node $x_i$ is defined as,  
\begin{equation}
	\delta_i=\{x_j | \arg\min\limits_{1 \le j \le n, j \ne i} dist(x_i,x_j)\}.
\end{equation}
And next, once more, $\delta_i^{(1)}$ is going to connect its nearest neighbor $\delta_i^{(2)}$, and so forth. Notice that, to avoid confusion caused by node's multiple nearest neighbors, a very small random variable $\varepsilon$ is applied to disturb the distance between nodes as $\mathfrak{d}_{ij} \gets \mathfrak{d}_{ij}+\varepsilon_{ij}$, where $\varepsilon_{ij} \ll \mathfrak{d}_{ij}$. Hence, denoting $x_i =\delta_i^{(0)}$, this searching process will produce a chain $\tau = \{(\delta_i^{(m)},\delta_i^{(m+1)})\}_{m=0}^h$. The searching process terminates when one of the below conditions meets: 

\begin{enumerate}
	\item $\delta_i^{(h-1)} = \delta_i^{(h+1)}$ and $\delta_i^{(h+1)} \in \Gamma$. It means that $\delta_i^{(h-1)}$ and $ \delta_i^{(h)}$ form a pair of RNNs as $ \delta_i^{(h)}$ and $\delta_i^{(h+1)}$ are the nearest neighbors of the $\delta_i^{(h-1)}$ and $ \delta_i^{(h)}$, respectively. Then we treat the chain $\tau$ as a new detected sub-MST;
	\item $\delta_i^{(h+1)} \notin \Gamma$. It means that $\delta_i^{(h+1)}$ has been searched before, which belongs to an existing sub-MST $\tau^*$. Thus, the chain $\tau$ needs to link to the $\tau^*$. 
\end{enumerate}

After that, all the nodes $\{\delta_i\}_{0}^{h+1}$ in $\tau$ will be removed from the candidate set $\Gamma$ as $\Gamma \gets \Gamma - \{\delta_i\}_{0}^{h+1}$. The construction procedure of sub-MSTs completes if and only if $\Gamma = \varnothing$, which means all the nodes are traversed. The corresponding pseudocode is shown in Algorithm 1.

\begin{algorithm}
\caption{\textsc{subMSTsCons}}  
	\begin{algorithmic}[1]   
	\Require Data: $X=\{x_i\}^n$
 	\Ensure $\mathcal{C}$ 
            	\State $\mathcal{C}\gets \varnothing$
            	\State $\Gamma \gets X$
            	\While {$|\Gamma|>0$}
            		\State $\tau \gets (V=\varnothing,E=\varnothing)$
            		\State $x_i \gets $\Call{RandomSelect}{$\Gamma$}
            		\While {Ture}
            		\State $\delta_i \gets \{x_j | \arg\min\limits_{1 \le j \le n, j \ne i}  dist(x_i,x_j)$ \}
            		\State $E \gets E \cup (x_i,\delta_i)$
            		\If{$\delta_i \in V$ {\bf or} $\delta_i \notin \Gamma$}  
            			\State $\Gamma \gets \Gamma - V$
            			\State {\bf break}
            		\EndIf
					\State $V \gets V \cup \delta_i$
            		\State $x_i \gets \delta_i$
            		\EndWhile
            	\State $\mathcal{C} \gets \mathcal{C} \cup \tau$
            	\EndWhile       
            	\State \Return {$\mathcal{C}$}
            
         	\end{algorithmic}
	
\end{algorithm} 

Fig. 1 illustrates the procedure of sub-MSTs construction for 10 two-dimensional data points which is a selected part of the entire data. Fig.1 (a) shows the distribution of the original data points and 10 selected ones in a dashed box; Fig.1~(b) shows a heatmap demonstrating the distances between the 10 points; Fig.1~(c)-(g) shows the construction process of the sub-MSTs, in which an initial chain-style sub-MST $\tau_1$ (with the nodes {\bf 2}, {\bf 3}, {\bf 4} and {\bf 5}) is detected in Fig.1~(c), and the other sub-MSTs \{(1,2)\}, \{(7,8),(8,4)\}, \{(10,8)\} are detected in Fig.1~(e), (f), (g), respectively. Meanwhile, a small sub-MST $\tau_2$ that only has two nodes {\bf 6}, {\bf 9} is detected in Fig.1~(d).   

We can easily deduce that a pair of RNNs is actually the area with the highest density in a cluster. Thereby, we treat the RNNs as the core of the cluster and, in the following sections, we will focus on how to discriminate the most representative node, i.e., the root, from the two nodes of the RNNs.

\begin{figure*}
\label{fig1}
\centering
\includegraphics[width=\textwidth]{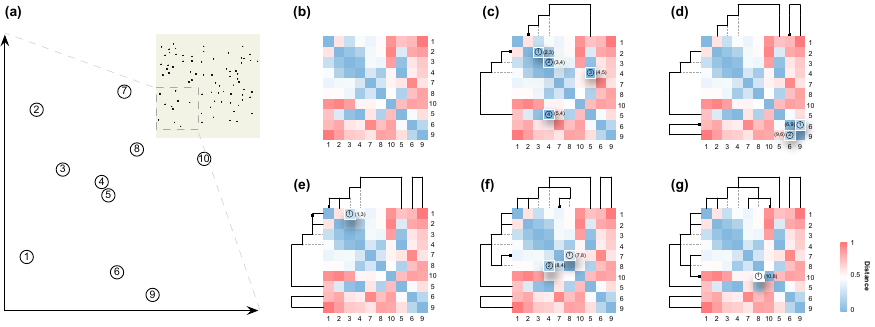}  
\caption{{\bf Illustration of the procedure to construct sub-MSTs.} (a) The distribution of the data points. (b) The distances heatmap of the 10 nodes. (c) Node {\bf 2} is randomly selected to start searching. Then, an initial chain-style sub-MST $\tau_1$ with the searching path \{(2,3), (3,4), (4,5), (5,4)\} could be derived, in which nodes {\bf 4} and {\bf 5} is a pair of RNNs. (d) Node {\bf 6} is randomly selected from the remainder \{1,6,7,8,9,10\}, and forms a small sub-MST $\tau_2$ with only two nodes {\bf 6} and {\bf 9}. (e) Node {\bf 1} is randomly selected from the updated candidate set \{1,7,8,10\}, and is linked to the existing sub-MST $\tau_1$ via the edge (1,3). (f) Node {\bf 7} is randomly selected from the remainder \{7,8,10\}, and forms a branch \{(7,8),(8,4)\} of the existing sub-MST $\tau_1$. (g) Node {\bf 10} is finally linked to the sub-MST $\tau_1$ via the edge (10,8).}

\end{figure*}

%
%

\subsection{Root detection in a cluster}
First of all, we need to explain why determining the root from a pair of RNNs is not trivial. According to our experiments and related literature, some simple options are available yet neither effective nor efficient. For example, randomly selecting one of the two nodes as the root might result in the elongated cluster being detected as shown in Fig. 2(d). On the contrary, the detected clusters shown in Fig. 2(c) are much more desired if we could select the root via a better tactic. Another simple try is to calculate the middle point between the two nodes as the root. But, this attempt will lead to an unbalanced cluster tree, and some additional prunings are required according to literature \cite{Xie 2020}. 

\begin{figure}[htb]
	\label{chain}
	\centering
	\includegraphics[width=3.5in]{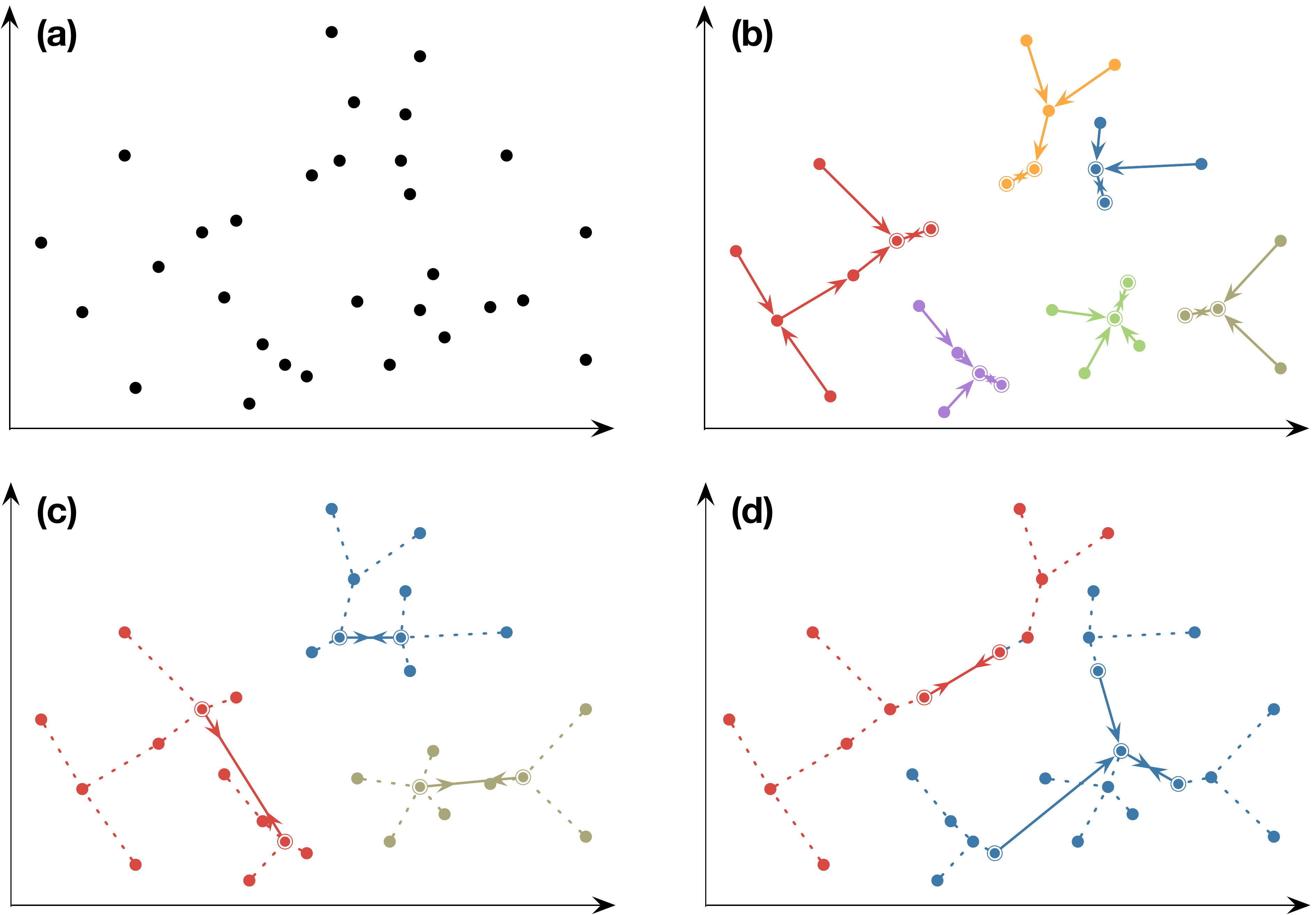}  
	\caption{{\bf Impacts of root selections.} (a) The distribution of data point. (b) The result of the first round of sub-MSTs construction. (c) The result of the second round of sub-MSTs construction based on the RNNs nodes' scores to select roots. (d) The result of the second round of sub-MSTs construction based on selecting roots randomly.}
\end{figure}

With respect to the generated graph partition $\mathcal{C} = (V,E)$ from Algorithm 1, we can use an adjacency matrix $\bf A$ to describe the construction of sub-MSTs more clearly, in which
\begin{equation}
a_{ij}= \left \{
\begin{aligned}
0 & , (x_i, x_j) \notin E \\
1 & , (x_i, x_j) \in E 
\end{aligned}
\right ..
\end{equation}
Furthermore, by introducing a relationship matrix $\bf R=A+A^T$, the RNNs in the all sub-MSTs could be easily detected as,
\begin{equation}
\Psi({\bf R}) = \{(x_i, x_j)|r_{ij}=2\},
\end{equation}
where $r_{ij}$ is an element in $\bf R$. Under the hypothesis that the reciprocal nearest neighbors are two possible representative data points, the next task is to determine which one in the RNNs is the most representative node, a.k.a. root. We put forward four measures to evaluate the RNNs data points, including: (1) degree-based index; (2) average neighbors' degree-based index; (3) path-based centrality index; (4) distance-based centrality index. 

Considering that a node with more neighbors usually gains more attention, we firstly use node degree to evaluate the RNNs nodes. The degree of the node $x_i$ is calculated by 
\begin{equation}
	d_i=\sum_{1\le j\le n} r_{ij},
\end{equation}
where the $r_{ij}$ is the element in relation matrix {\bf R}. From a perspective of topology structure, nodes with larger degree in the cluster obviously are more important than those with smaller degree. 

Next, we consider that a node is more important if its neighbors have a larger degree, which means this node is located at the center of the dense area of a graph. Thus, we use the neighbors' average degree of a node to measure its representativeness by
\begin{equation}
	\bar{d}_i=\frac{1}{d_i}\sum\limits_{r_{ij}=1,1\le u\le n}r_{ju}.
\end{equation}
This measure quantifies the averaged importance of neighbors associated with current node $x_i$. 

In addition, to further measure the centrality of the RNNs, we also propose  two other indices including path-based centrality index and distance-based centrality index. The definition of path-based centrality index is given by
\begin{equation}
	c_i=\frac{1}{|\tau|}\sum_{x_j \in \tau} s_{ij}.
\end{equation}
where $|\tau|$ is the number of nodes in the cluster $\tau$; $s_{ij}$ is the shortest path length between nodes $x_i$ and $x_j$. A smaller $c_i$ suggests that the shape of the cluster $\tau$ started from the node $x_i$ is more like a star. On the other hand, the distance-based centrality index is a simply improved version of the path-based centrality index, which is defined as,
\begin{equation}
	c^*_i=\frac{1}{|\tau|}\sum_{x_j \in \tau} \frac{\mathfrak{d}_{ij}}{s_{ij}},
\end{equation}
where $\mathfrak{d}_{ij}$ is the distance between nodes $x_i$ and $x_j$. 

More comprehensively, we come up with a hybrid index $\psi$ that integrates all of the above indices. Given a pair of RNNs $x_i$ and $x_j$, the combined score is calculated by, 

\setlength{\arraycolsep}{0.0em}
\footnotesize
\begin{equation}
	\psi_i=\frac{1}{4}\left(\frac{d_i}{d_i+d_j}+\frac{\bar{d}_i}{\bar{d}_i+\bar{d}_j}+(1-\frac{c_i}{c_i+c_j})+(1-\frac{c^*_i}{c^*_i+c^*_j})\right),
\end{equation}
\normalsize
obviously, $\psi_i + \psi_j = 1$. In addition, we also provide a simplified version of the hybrid index that only considers average neighbor's degree-based index and distance-based centrality, which is calculated by,
\begin{equation}
	\psi_i^*=\frac{1}{2}\left(\frac{\bar{d}_i}{\bar{d}_i+\bar{d}_j}+(1-\frac{c^*_i}{c^*_i+c^*_j})\right).
\end{equation}
If condition $\psi_i > \psi_j$ or $\psi_i^* > \psi_j^*$ meets, compared to $x_j$, $x_i$ can be decided as the root in the sub-MST. Thereby, by calculating Eq. (8) or (9),  we can easily determine the root set $\Lambda = \{r_i\}_1^m$ in all $m$ sub-MSTs, each of which is nested in a tree-style structure. In a higher level of clustering, these roots could be regarded as new nodes to form new sub-MSTs, i.e., putting all the root nodes in $\Lambda$ as a new candidate set of nodes, and then returning to the searching process to link the roots and construct tree-style sub-MSTs. As a result, such iterative process will reduce the number of sub-MSTs level-by-level and generate the final cluster tree. Of course, the iteration can also cease when a given granularity of clustering result reaches. In addition, if we want to generate a fixed number of clusters, i.e., $K$, we might not get the clusters with the exact number of $K$ due to the hierarchical structure of clustering. In this case, if having recorded the number of roots at last iteration as $k^+$, we calculate the $k^-$ roots at current iteration via the obtained relationship matrix $\bf R$.

\begin{equation}
 k^-= \frac{|\Lambda|}{2} = \frac{1}{2}\sum_{i,j} \lfloor\frac{1}{2}r_{ij}\rfloor .	
\end{equation}
If the condition $k^+>K>k^-$ holds, we treat $k^+$ as the final number of clusters. On the other hand, if we want to obtain the expected number of clusters, we can also add a simple step by linking the $k^+-K$ closest roots to merge clusters, and then we will get the $K$ final clusters as desired. The pseudocode of this iteration process and the scoring process are described in Algorithm 2 and Algorithm 3, respectively, in which the boundary-based score $\zeta$ will be introduced in the next section. 

\begin{algorithm}
\caption{\textsc{Iteration}}
	\begin{algorithmic}[1]
     	\Require Data: $X=\{x_i\}^n$, \#Cluster: $K=1$
    	\Ensure $Labels$  

                \State $\Lambda \gets X$  
                \State $\mathcal{C} \gets \varnothing$
                \While {$|\Lambda|>K$} 
                	\State $\mathcal{C}^* \gets$ \Call{subMSTsCons}{$\Lambda$}
                	\State $\mathcal{C} \gets \mathcal{C} \cup \mathcal{C}^*$
                	\State ${\bf A} \gets$ \Call{Translate}{$\mathcal{C}^*$}
                	\State ${\bf R = A+A^T}$
                	\State $\Lambda \gets$  \Call{Scoring}{$\Psi({\bf R}), \bf{R}$}  
                \EndWhile  
                \State $Labels \gets$ \Call{Labeling}{$\Lambda, \mathcal{C}$}
           	\end{algorithmic}
	
\end{algorithm}   

\begin{algorithm}
\caption{\textsc{Scoring}}
	\begin{algorithmic}[1]
     	\Require Data: $\Psi({\bf R}), {\bf R}$
    	\Ensure $\Lambda$  

                \State $\Lambda \gets \varnothing$  
                \For {$(x_i,x_j)$ {\bf in} $\Psi(\bf R)$}
                	\If {$\psi_i \neq \psi_j$}
                		\State $v_i=\psi_i$; $v_j=\psi_j$
                		
                	\Else
                		\State $v_i=\zeta_i$; $v_j=\zeta_j$
                	\EndIf 
                	\State $\Lambda \gets \Lambda + \arg\max (v_i, v_j)$
                \EndFor

           	\end{algorithmic}
	
\end{algorithm}   

Figure 3 illustrates the self-similarity construction procedure of the cluster tree. As in the amplifier shown at the upper-left corner of Fig. 3, we get the roots of the exampled sub-MSTs shown in Fig. 1. For the sake of clarity, the sub-MSTs obtained in the first iteration are folded and represented by their roots. After three iterations, a complete cluster tree for the full data is achieved eventually.

\begin{figure}
\centering
  \includegraphics[width=3.5in]{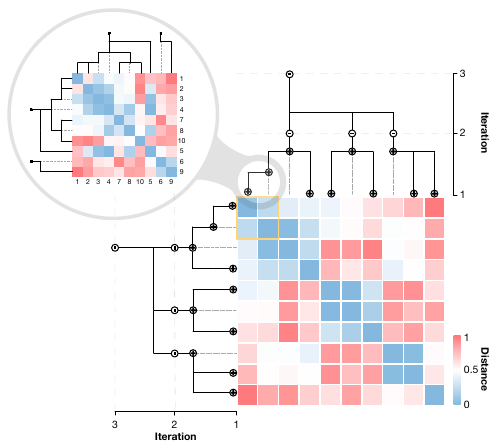}  
\caption{{\bf Illustration of the iteration and the further construction procedure.} The amplifier shows the unfolded sub-MSTs we get in Fig. 1, in which two roots are confirmed based on their scores. Folding the 10 sub-MSTs that are obtained in the first iteration, the roots are used as a new candidate set of nodes to construct sub-MSTs again. Finally, 3 roots are obtained and only one sub-MST is constructed. It means the full cluster tree is generated eventually.}
\end{figure}

\subsection{ Extended method for anomalous scoring on RNNs}
\begin{figure}
	\label{fig4}
	\centering
	\includegraphics[width=3in]{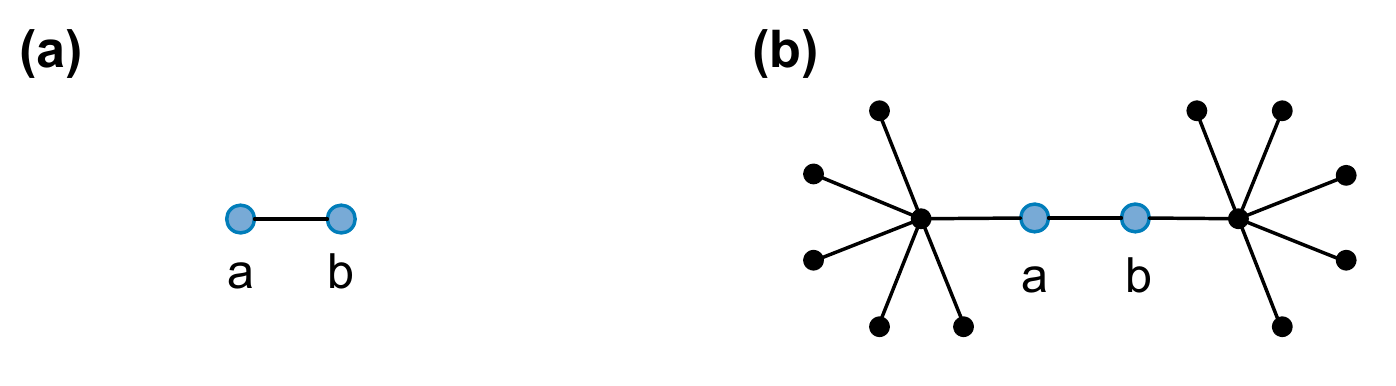}  
	\caption{{\bf Two cases  of anomalous scoring on RNNs data points.}}
\end{figure}

In most cases, we can simply use the measures proposed in the previous section to uniquely detect the root for a given pair of RNNs in a cluster in terms of their different scores. However, despite appearing very rare in practice, the two RNNs in a cluster are going to be scored equally in two special cases. As shown in Fig.~4(a), if nodes $a$ and $b$ are mutual neighbors and have no other neighbors existed, they apparently will get the same scores. The second case is that, when owning the identical neighborhood shown in Fig.~4(b), they would also get same scores. We call them two cases of anomalous scoring on RNNs. Thus, we propose an extended detection method to address this issue, which is on the basis of the following assumption: 

In a pair of RNNs points, the one who is closer to the data boundary would be more important and thus could be chosen as the root, in which, the data boundary here is defined as below.

\begin{myDef}
	Data Boundary. The data boundary is a closed curve formed by connecting some finite $\sigma$ pairs of data points, in which each pair of them has the farthest distance.
\end{myDef}

Commonly, a data boundary could be detected through a procedure as follows: we firstly find out $\sigma$ pairs of points having the top-$\sigma$ farthest distances. Note that the $\sigma$ pairs data points are $2\sigma$ unique points. And then, we link the selected points by $k$-nearest-neighbor method (here $k=2$). This straightforward method, however, is very expensive for large data set due to the high time-consumption of ranking the pairwise distances. Therefore, we propose a sampling strategy to help us accelerate the process of data boundary detection, which is on the basis of the property Theorem 1.

\begin{theorem}
Given a dataset $X=\{x_i\}^n$ and a randomly selected data point $x_i$ therein, if the furthest point of $x_i$ is $\eta_i^{(1)}$, and the furthest point of $\eta_i^{(1)}$ is $\eta_i^{(2)}$,  $\eta_i^{(1)}$ and $\eta_i^{(2)}$ must be a pair of points on the data boundary.
\end{theorem}
\begin{proof}
 If $\eta_i^{(1)}$ is not on the data boundary, then $\eta_i^{(1)}$ can only be within the data boundary. Meanwhile, there must be at least one another point on the data boundary, denoted as $\eta_i^{(1)*}$, satisfying $x_i$, $\eta_i^{(1)}$ and $\eta_i^{(1)*}$ on a straight line. There must be satisfying,
 \begin{equation}
	dist(x_i, \eta_i^{(1)})<dist(x_i, \eta_i^{(1)}) + dist(\eta_i^{(1)},\eta_i^{(1)*}), \nonumber
\end{equation}
where,
\begin{equation}
 dist(x_i, \eta_i^{(1)}) + dist(\eta_i^{(1)},\eta_i^{(1)*})=dist(x_i, \eta_i^{(1)*}).\nonumber	
\end{equation}
This contradicts that $\eta_i^{(1)}$ is the furthest point away from $x_i$; therefore, $\eta_i^{(1)}$ must be just on the data boundary. For the same reason, it can be proved that $\eta_i^{(2)}$ must also be on the data boundary. Theorem 1 is proved.
\end{proof}

As a result, to detect the data boundary efficiently, we can sample the pairwise points on the boundary by randomly selecting a set of nodes $\{x_i\}^n$ and, for each $x_i$, calculating its 1st-order and 2rd-order furthest points, $\eta_i^{(1)}$ and $\eta_i^{(2)}$, respectively. The detailed description of the processing is shown in Algorithm 4.

\begin{algorithm}
\caption{\textsc{BoundryDetection}}
	\begin{algorithmic}[1]
     	\Require Data: $X=\{x_i\}^n$, \#Boundary: $2\sigma$
    	\Ensure $\mathfrak{B}$  
    		\State $\mathfrak{B} \gets \varnothing$; $\mathfrak{B}^* \gets \varnothing$
    		\For {$x_i$ {\bf in} \Call{RandomSelecte}{$X, \sigma$}}
    			\State $\eta_i^{(1)} \gets \{x_j | \arg\max\limits_{\overset{1 \le j \le n,}{\overset{j \ne i,}{x_j \notin \mathfrak{B}^*}}}  dist(x_i,x_j)\}$ 
    			\State $\eta_i^{(2)} \gets \{x_j | \arg\max\limits_{\overset{1 \le j \le n,}{\overset{\eta_i^{(1)}  \ne x_j,}{x_j \notin \mathfrak{B}^*}}}  dist(\eta_i^{(1)},x_j)\}$ 
    			\State $\mathfrak{B} \gets \mathfrak{B} + \{(\eta_i^{(1)}, \eta_i^{(2)})\}$
    			\State $\mathfrak{B}^* \gets \mathfrak{B}^* + \eta_i^{(1)} + \eta_i^{(2)}$
    		\EndFor
                  \end{algorithmic}
	
\end{algorithm}

And then, by employing the decision Theorems 2 and 3, it makes sense to use the sampled $\sigma$ pairs of points to measure the closeness score between a data point $\mathfrak{r}$ and the data boundary as,
\begin{equation}
	\zeta = \frac{1}{\sigma}\sum_{i=0}^{\sigma} |dist(\mathfrak{r},\eta_i^{(1)})-dist(\mathfrak{r},\eta_i^{(2)})|,
\end{equation}
where, the $\eta_i^{(1)}$ and $\eta_i^{(2)}$ are $i^{th}$ pair of data points on the data boundary, Theorem 2 ensures the scores over the RNNs are different, and Theorem 3 ensures the correlation between the scores over RNNs and their distances to the data boundary.

\begin{theorem}
Three boundary points that are not on the same line can be used to score a pair of RNNs differently.
\end{theorem}
\begin{proof}
Suppose that there is a pair of RNNs, denoted as $(\mathfrak{r}, \mathfrak{r}')$, and three random boundary points $\{\mathfrak{b}_1, \mathfrak{b}_2, \mathfrak{b}_3\}$ which are not on a same line (there are three pairwise combinations, $(\mathfrak{b}_1, \mathfrak{b}_2)$, $(\mathfrak{b}_1, \mathfrak{b}_3)$, $(\mathfrak{b}_2, \mathfrak{b}_3)$). Firstly, calculate the distances between $\mathfrak{r}$ and the three data points $(\mathfrak{b}_1, \mathfrak{b}_2, \mathfrak{b}_3)$, noted as $\mathfrak{d}_1$, $\mathfrak{d}_2$ and $\mathfrak{d}_3$, respectively, and compute,
\begin{equation}
\zeta=\frac{1}{{\bf C}(3,2)}\sum_{1\le j<i\le 3}|\mathfrak{d}_i-\mathfrak{d}_j|,\nonumber
\end{equation}
as the score of $\mathfrak{r}$. Since $\{\mathfrak{b}_1, \mathfrak{b}_2, \mathfrak{b}_3\}$ is not on a same line, at least one of the terms $|\mathfrak{d}_{i}-\mathfrak{d}_{j}|$ must not be 0, therefore the score $\zeta$ does not equal 0. Likewise, for $\mathfrak{r}^*$, a score $\zeta'$ can be calculated which is also not equal to 0. Since $\mathfrak{r}$ and $\mathfrak{r}'$ are different data points, $\zeta \neq \zeta'$ must be satisfied, and Theorem 2 is proved.
\end{proof}

\begin{theorem}
Given a pair of RNNs $(\mathfrak{r}, \mathfrak{r}')$ and their boundary-based score $(\zeta,\zeta')$. If $\zeta>\zeta'$, it suggests that point $\mathfrak{r}$ is more nearer to the boundary than $\mathfrak{r}'$.
\end{theorem}
\begin{proof}
Without loss of generality, if $\mathfrak{r}$ is approaching to one of the three boundary points $\{\mathfrak{b}_1,\mathfrak{b}_2,\mathfrak{b}_3\}$, such as $\mathfrak{b}_1$, the terms $\{|\mathfrak{d}_2-\mathfrak{d}_1|,|\mathfrak{d}_3-\mathfrak{d}_1|\}$ will be increasing and the term $|\mathfrak{d}_3-\mathfrak{d}_2|$ would not change much (here, we approximately assume that the term $|\mathfrak{d}_3-\mathfrak{d}_2|$ is unchanged). When $\mathfrak{r}$ arrives at the boundary point $\mathfrak{b}_1$, i.e., $\mathfrak{r}=\mathfrak{b}_1$ or $\mathfrak{d}_1=0$, then the score,
\begin{equation}
\zeta=\frac{1}{{\bf C}(3,2)}\left(\mathfrak{d}_2+\mathfrak{d}_3+|\mathfrak{d}_3-\mathfrak{d}_2|\right),\nonumber
\end{equation}
also reaches one of the maximums of the scores over $\mathfrak{r}$. Therefore, a score $\zeta$ can measure how close $\mathfrak{r}$ is to the data boundary. If $\zeta>\zeta'$, $\mathfrak{r}$ will be more approaching to the boundary than $\mathfrak{r}'$. Theorem 3 is proved.
\end{proof}

\section{Experiments}

\subsection{Data sets}

 As an open accessed data repository, UCI database \cite{Lichman 2013} is widely used in the communities of machine learning and data mining. In this study, fifteen real-world data sets with trustworthy labels selected from the UCI database are applied to test the performance of the proposed algorithm. The basic information of these adopted data sets is summarized in Table I.

\begin{table}
\footnotesize
\caption{{\bf Basic information of the fifteen selected UCI data sets.}}
\begin{center}
\begin{tabular}{|c|ccc|}
\hline
\textbf{DataSet}&\textbf{\#Samples}&\textbf{\#Features}&\textbf{\#Class}\\ 
\hline
iris&150&4&3\\
sonar&208&60&2\\
glass&214&9&2\\
ecoli&336&7&8\\
ionosphere&351&60&2\\
synthetic control&600&60&6\\
vehivel&846&18&4\\
mfeat-fourier&2000&76&10\\
mfeat-karhunen&2000&64&10\\
mfeat-zerike&2000&47&10\\
segment&2310&19&7\\
wavefore-5000&5000&40&3\\
optdigits&5630&64&9\\
letter&20000&16&26\\
avilia&20867&34&12\\

\hline
\end{tabular}
\end{center}
\end{table}

\subsection{Evaluation Index}

In the experiments, we apply the well-known Rand Index \cite{Rand 1971} and Normalized Mutual Information (short for NMI) \cite{Cover 2005} to compare the results.

\subsubsection{Rand Index} Rand Index measures the similarity between two independent partitions over a dataset. Given a dataset with $n$ items $X=\{x_i\}^n$ and its two partitions of the dataset, i.e., the real partition $\mathcal{C}$ and the algorithm-produced partition $\mathcal{C}^*$, the Rand Index is defined as,
\begin{equation}
   RI=\frac{a+b}{{\bf C}(n,2)},
\end{equation}
where $a$ is the number of pairwise data points that are assigned to a same cluster in both $\mathcal{C}$ and $\mathcal{C}^*$; $b$ is the number of pairwise data points that are assigned to different clusters in both $\mathcal{C}$ and $\mathcal{C}^*$; ${\bf C(\cdot)}$ is the combination formula.

\subsubsection{Normalized Mutual Information} NMI measures how much information about a variable is contained in the target variable. Given two random variables, $X$ and $Y$, and the joint distribution $p(x,y)$, marginal distribution $p(x)$ and $p(y)$, respectively. The Mutual Information, $MI(X;Y)$, is defined as, 
\begin{equation}
   MI(X;Y)=\sum_{x\in X}\sum_{y\in Y}p(x,y)\log \frac{p(x,y)}{p(x)p(y)}.
\end{equation}

However, the value of Mutual Information is week in showing the advantage of algorithms intuitively. Hence, Mutual Information is usually assigned to $[0,1]$, called Normalized Mutual Information that is calculated as below,

\begin{equation}
   NMI(X;Y)=\frac{2MI(X;Y)}{H(X)+H(Y)},
\end{equation}
where $H(\cdot)$ is the information entropy, $H(X)=-\sum_{x\in X}p(x_i)\log p(x_i)$, and $H(Y)$ in the same way.

\subsection{Experimental results}
\subsubsection{Comparison of different indices for scoring RNNs} We first compare the performances of the scoring measures including four independent indices and two hybrid indices. Here, without loss of generality, we just pick out six data sets from the fifteen ones, which are from diversified domains.  As illustrated in Fig. 5, the performances are shown by using the box diagram which contains the upper-bound, lower-bound, mean, one-quarter and three-quarter. Generally, the degree-based index (noted as $d$) and the path-based centrality index (noted as $c$) is less competitive than average neighbor's degree-based index (noted as $\bar{d}$) and distance-based centrality index (noted as $c^*$). More clearly, the detailed results are also presented in Table II. Both the lowerbound and mean of degree-based index and path-based centrality index are significantly lower than the two other indices. Besides, the comparison between the hybrid index (comprising four independent indices, noted as $\psi$) and the simplified hybrid index (comprising only two indices, i.e., $\bar{d}$ and $c^*$, noted as $\psi^*$) shows that the simplified one gets higher lowerbound and larger means,   which means that the simplified hybrid index gets better performance and requires less computation. Hence,  in the following experiments, we prefer to use the simplified hybrid index to evaluate the RNNs data points.

\begin{figure}
\centering
  \includegraphics[width=0.8\textwidth]{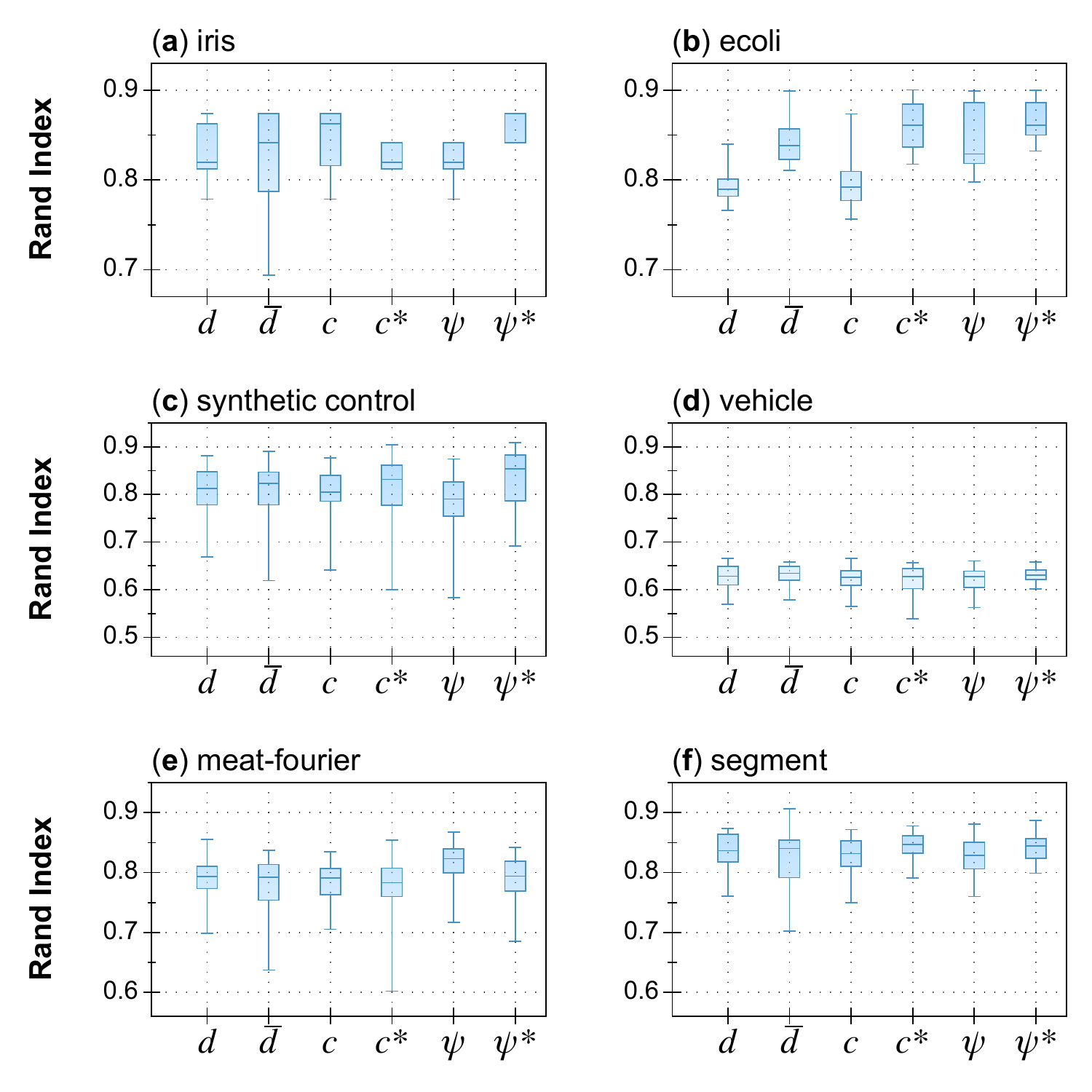}
\caption{{\bf The comparison of different indices for anomalous scoring on RNNs on six UCI data sets.} The X-axis shows the box diagram of 100 results corresponding to the different indices.}
\end{figure}

\begin{table*}
	\caption{The lower-bound and mean of different indices on six UCI benchmarks.}
	\small
	\begin{center}
	\begin{tabular}{|c|c|cccccc|}
	\hline
	&\multirow{2}{*}{\bf Datasets}&\multirow{2}{*}{\bf $\bar{d}$}&\multirow{2}{*}{\bf $c$}&\multirow{2}{*}{\bf $c^*$}&\multirow{2}{*}{\bf $d$}&\multirow{2}{*}{\bf $\psi$}&\multirow{2}{*}{\bf $\psi^*$}\\
	&&&&&&&\\
	\hline
	\multirow{7}{*}{\rotatebox{90}{\bf MEAN}}&iris&0.8202&{\bf 0.8481}&0.8242&0.8339&0.8231&0.8426\\
	&ecoli  &0.8425	&0.7949	&0.8606	&0.7922	&0.8495	&{\bf 0.8653}\\
    &synthetic control 	&0.8076	&0.7998	&0.8183	&0.8067	&0.7853	&{\bf0.8349}\\
    &vehicle	&{\bf0.6284}	&0.6229	&0.6181	&0.6247	&0.6218	&0.6234\\
    &meat-fourier	&0.7768	&0.7843	&0.7777	&0.7899	&{\bf0.8182}	&0.7876\\
    &segment	&0.8213	&0.8237	&{\bf0.8383}	&{0.8363}	&0.8221	&0.8353\\\cline{2-8}
    &{\bf Means}	&0.7828	&0.7790	&0.7895	&0.7806	&0.7867	&{\bf 0.7982}\\
	\hline
	\multirow{7}{*}{\rotatebox{90}{\bf MIN}}&iris	&0.6937	&0.7785	&0.8121&0.7785	&0.7785	&{\bf0.8415}\\
    &ecoli	&0.8104	&0.7561	&0.8176	&0.7660	&0.7974	&{\bf0.8320}\\
    &synthetic control 	&0.6192	&0.6410	&0.6002	&0.6692	&0.5830	&{\bf0.6917}\\
    &vehicle	&0.5782	&0.5709	&0.5453	&0.5630	&0.5586	&{\bf 0.6022}\\
    &meat-fourier	&0.6370	&0.7051	&0.6019	&0.6984	&{\bf0.7165}	&0.6849\\
    &segment	&0.7024	&0.7441	&0.7931	&0.7604&0.7625	&{\bf 0.8042}\\\cline{2-8}
    &{\bf Means}	&0.6735	&0.6993	&0.6950	&0.7059 &0.6994	&{\bf 0.7427}\\

	\hline
	\end{tabular}
	\end{center}
\end{table*}

\subsubsection{Comparison between random strategy and the boundary-based strategy on anomalous scoring over RNNs} As a simple strategy mentioned in the previous section, randomly selecting a point from the RNNs is seemingly also able to handle the issue of anomalous scoring over RNNs. In order to verify the superiority of the proposed boundary-based strategy, we give some comparisons between the two strategies as following. Note that, to obtain a balanced performance, we  just sample $\log n$ pairs of boundary points by following Theorem 1. From the results of the boundary-based method shown in Fig. 6, we can observe that the difference between the values of upper bound and the lower bound for the metric Rand Index is significantly reduced, that is to say, the robustness of the boundary-based method is much better than the random strategy. Meanwhile, it also shows that both the mean value and the lower bound of Rand Index for the boundary-based method are higher than those for the random one in most datasets as shown in Table III. Therefore, it can be concluded that the boundary-based strategy is a great improvement over the random strategy.

\begin{figure}
\label{fig6}
\centering
  \includegraphics[width=\textwidth]{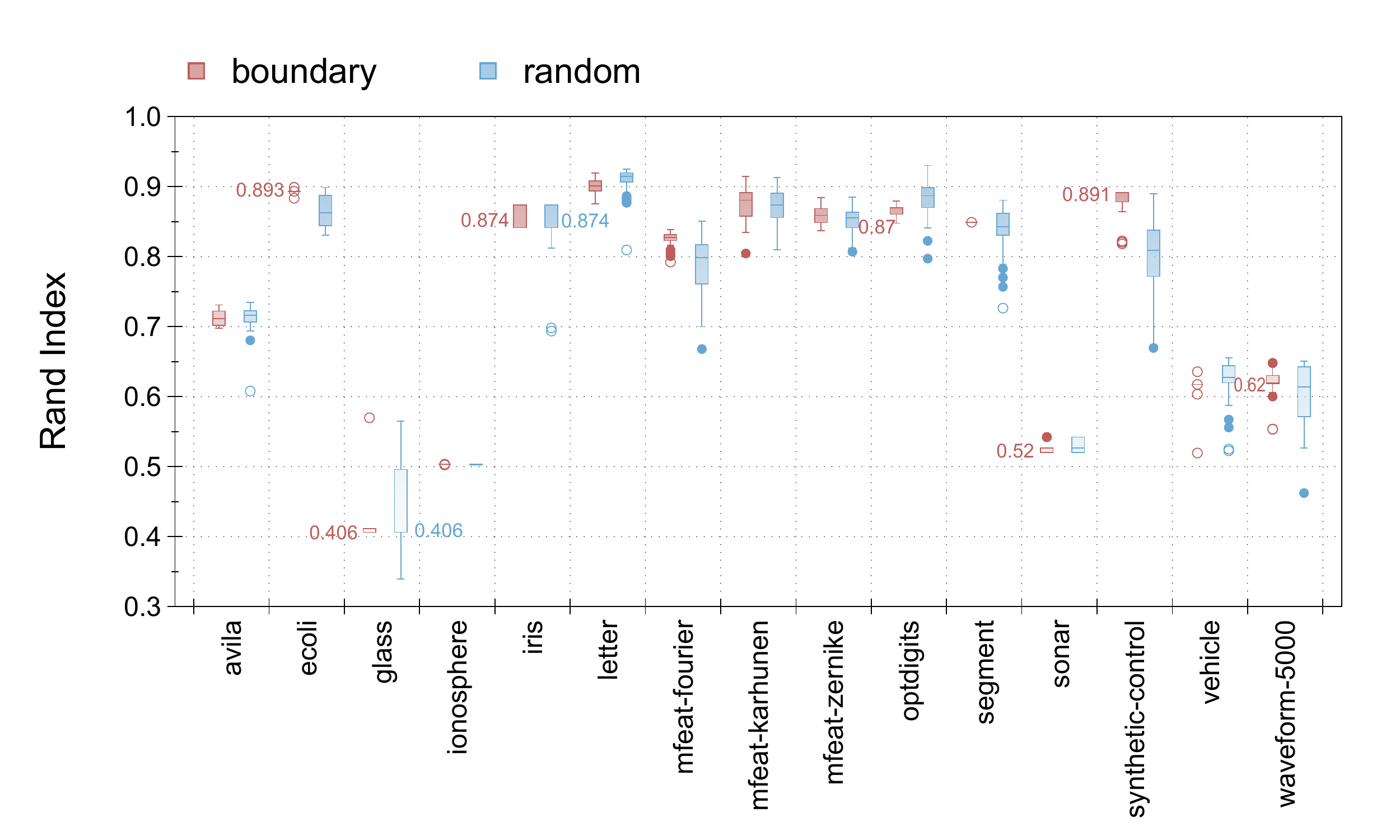}  
\caption{{\bf Comparison between random strategy and boundary-based strategy on fifteen UCI benchmarks.} The X-axis corresponds to different data sets and Y-axis shows the box diagram of 100 results (some of the boxes are marked with the mean values as their mean symbols are hard to be observed in the boxes). The circles are the noisy outcome that appears once or twice. }
\end{figure}

\begin{table}
	\caption{Comparison results of random strategy and boundary-based strategy for anomalous scoring on RNNs.}
	\small

	\begin{center}
	\begin{tabular}{|c|cccc|}
	\hline
	\multirow{3}{*}{\bf Datasets}&\multicolumn{2}{c}{\bf Random}&\multicolumn{2}{c|}{\bf Boundary-based}\\
	\cline{2-5}
	&\multirow{2}{*}{\bf means}&{\bf lower}&\multirow{2}{*}{\bf means}&{\bf lower}\\
	&&{\bf bound}&&{\bf bound}\\
	\hline
	{ecoli}&0.8665&0.8309&0.8936&{\bf 0.8834}\\
	{glass}&0.4247&0.3396&0.4141&{\bf 0.4063}\\
	{iris}&0.8376&0.6937&0.8621&{\bf 0.8415}\\
	{sonar}&0.5287&0.5200&0.5251&{\bf 0.5201}\\
	{vehivel}&0.6236&{\bf  0.5227}&0.6112&0.5195\\
	{ionosphere}&0.5029&{\bf 0.5025}&{ 0.5035}&{\bf 0.5025}\\
	{segment}&0.8385&0.7265&{ 0.8491}&{\bf 0.8491}\\
	{synthentic control}&0.8028&0.6695&{ 0.8789}&{\bf 0.8183}\\
	{mfeat-fourier}&0.7867&0.6680&{ 0.8256}&{\bf 0.7923}\\
	{mfeat-karhunen}&0.8712&{\bf 0.8100}&{ 0.8747}&0.8046\\
	{mfeat-zerike}&0.8510&0.8072&0.8510&{\bf 0.8372}\\
	{wavefore-5000}&0.6048&0.4623&{ 0.6224}&{\bf 0.5533}\\
	{ optdigits}&{ 0.8834}&0.7972&0.8671&{\bf 0.8479}\\	
	{ letter}&{ 0.9044}&0.8095&0.9005&{\bf 0.8757}\\	
	{ avilia}&{0.7125}&0.6079&0.7121&{\bf 0.6978}\\
	
	\hline
	{\bf Means}&0.7360&0.6512&{\bf 0.7461}&{\bf 0.7116}\\

	\hline
	\end{tabular}
	\end{center}
\end{table}

\subsubsection{Comparison with benchmark algorithms} We then compare the proposed algorithm (simplified hybrid index) with three classical hierarchical clustering algorithm and the state-of-the-art methods PERCH \cite{Kobren 2017} and RSC \cite{Xie 2020}. Because the CURE, Chameleon and RSC are parameter-sensitive, in this experiment, default parameters which were recommended by the literature \cite{Guha 2001, Karypis 1999, Xie 2020} are adopted in the two algorithms. As illustrated in Fig. 7 and Fig. 8, the proposed algorithm SRSC\footnote{Algorithm implementation codes and related data resources are available at https://github.com/wenboxie0211/SRSC.git.} (short for {\bf S}cored {\bf R}NN{\bf s} based {\bf C}lustering) is overall better than other benchmarks, as the result curve, marked in red, is above other ones in most cases. 

\begin{figure*}
\label{fig7}
\centering
  \includegraphics[width=\textwidth]{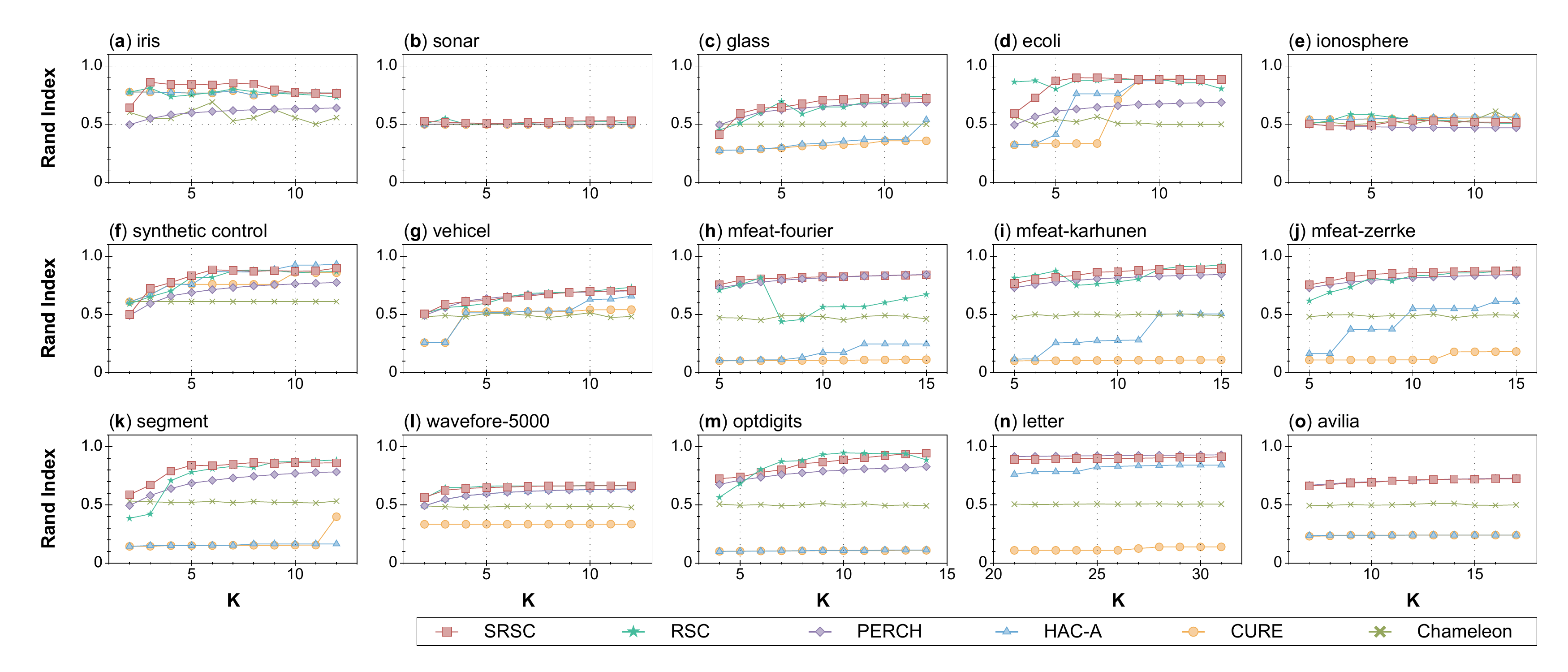}  
\caption{{\bf The Rand Index of the proposed algorithm and benchmark algorithms on fifteen UCI datasets.} Each point represents the average Rand Index of the corresponding algorithm based on the number of  clusters $K$. Notice that the range of the $X$-axis depends on datasets'  actual class-label amount, i.e., $\mathcal{K}$: if $\mathcal{K} \ge 7$, $K \in [\mathcal{K}-7,\mathcal{K}+7]$; otherwise, $K \in [2, 12]$. }
\end{figure*}

\begin{figure*}
\label{fig8}
\centering
  \includegraphics[width=\textwidth]{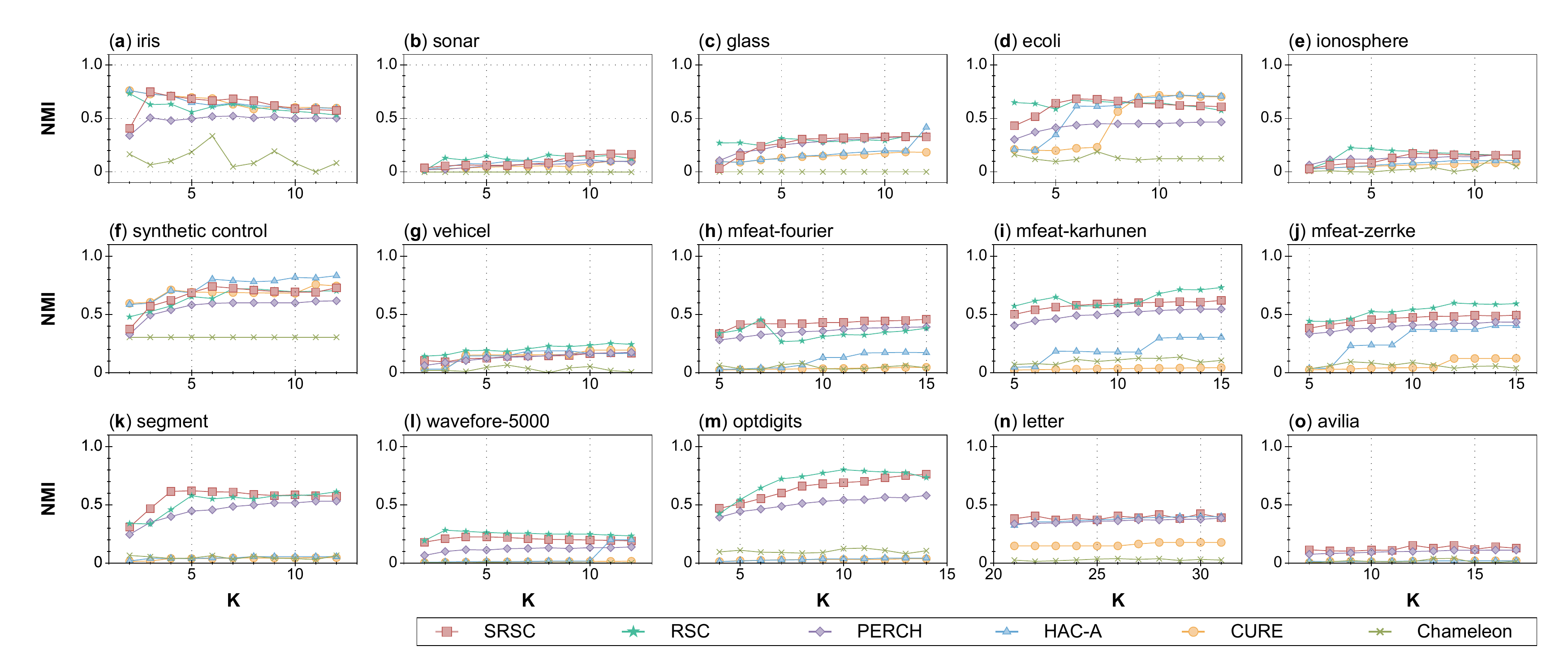}  
\caption{{\bf The Normalized Mutual Information of the proposed algorithm and benchmark algorithms on fifteen UCI datasets.} Each point represents the average Normalized Mutual Information of the corresponding algorithm based on the number of  clusters $K$. Notice that the range of the $X$-axis depends on datasets'  actual class-label amount, i.e., $\mathcal{K}$: if $\mathcal{K} \ge 7$, $K \in [\mathcal{K}-7,\mathcal{K}+7]$; otherwise, $K \in [2, 12]$. }
\end{figure*}

More clearly, we compare the details of ground-truth on each data set in Table IV (Rand Index) and Table V (NMI). Table IV shows that the proposed algorithm gets highest scores of Rand Index on nine data sets out of the fifteen data sets, and the mean of Rand Index for the proposed algorithm is also the best. Similarly, Table V shows that the proposed algorithm gets highest scores of NMI on eight data sets out of the fifteen data sets as well, and also, the mean of NMI for the proposed algorithm is the best. Note that due to insufficient memory, the baseline, RSC, cannot complete the clustering process on $letter$ and $avilia$, resulting in some of the null values in Table IV and V as well as the absence in Table VI and Fig. 9.

\begin{table}
	\caption{Rand Index of the proposed algorithm and benchmark algorithms on the ground-truth of UCI data sets.}
	\tiny
	\begin{center}
	\begin{tabular}{|c|cccccc|}
	\hline
	{\bf Datasets}&{\bf HAC-A}&{\bf CURE}&{\bf Chameleon}&{\bf PERCH}&{\bf RSC}&{\bf SRSC}\\
	\hline
    { iris}&0.7771&0.7771&0.7598&0.5490&0.8120&{\bf 0.8621}\\	
    {sonar}&0.4993&0.4993&0.4999&0.5004&0.5032&{\bf 0.5251}\\
	{glass}&0.2765&0.2765&{\bf 0.5011}&0.4509&{0.4948}&0.4141\\	
    {ecoli}&0.7618&{0.7083}&0.5655&0.6593&0.8854&{\bf 0.8936}\\
	{ ionosphere}&{\bf 0.5401}&{\bf 0.5401}&0.5158&0.5011&0.5035&0.5035\\
	{ synthentic control}&{ 0.8703}&0.7600&0.6048&0.7146&0.8200&{\bf 0.8789}\\
    { vehivel}&0.5154&0.5244&0.4904&{\bf 0.6180}&0.5325&{ 0.6112}\\
	{ mfeat-fourier}&0.1730&0.1068&0.5052&{ 0.8159}&0.5651&{\bf 0.8256}\\
	{ mfeat-karhunen}&0.2786&0.1067&0.5198&0.8161&0.7815&{\bf 0.8747}\\
	{ mfeat-zerike}&0.5499&0.1115&0.4943&0.8152&0.8340&{\bf 0.8510}\\
    { segment}&0.1531&0.1531&0.5330&0.7317&0.8311&{\bf 0.8491}\\
	{ wavefore-5000}&0.3350&0.3335&0.4931&0.5479&{\bf 0.6489}&{0.6224}\\
	{ optdigits}&0.1092&0.1064&0.5291&0.7893&{\bf 0.9322}&{0.8671}\\	
	{ letter}&0.8312&0.1096&0.5064&{\bf 0.9237}&--&0.9005\\	
	{ avilia}&0.4608&0.2402&0.5171&{ 0.7114}&--&{\bf 0.7121}\\
	
	\hline
	{\bf Means}&0.4754&0.3569&0.5357&0.6792&--&{\bf 0.7461}\\
	\hline
	\end{tabular}
	\end{center}
\end{table}

\begin{table*}
	\caption{NMI of the proposed algorithm and benchmark algorithms on the ground-truth of UCI data sets.}
	\tiny
	\begin{center}
	\begin{tabular}{|c|cccccc|}
	\hline
	{\bf Datasets}&{\bf HAC-A}&{\bf CURE}&{\bf Chameleon}&{\bf PERCH}&{\bf RSC}&{\bf SRSC}\\
	\hline
    { iris}&0.7284&0.7284&0.0657&0.5069&0.6302&{\bf 0.7498}\\	
    {sonar}&0.0209&0.0209&0.0025&0.0198&0.0095&{\bf 0.0368}\\
	{glass}&0.0951&0.0951&0.0012&0.1062&{\bf 0.2705}&0.0309\\	
    {ecoli}&0.6211&0.5641&0.1283&0.4498&0.6501&{\bf 0.6652}\\
	{ ionosphere}&0.0259&0.0259&0.0055&{\bf 0.0643}&0.0268&0.0275\\
	{ synthentic control}&{\bf 0.8038}&0.6909&0.3049&0.5963&0.6377&0.7416\\
    { vehivel}&0.1437&0.1510&0.0108&0.1034&{\bf 0.1888}&0.1216\\
	{ mfeat-fourier}&0.1302&0.0348&0.0337&0.3581&0.3113&{\bf 0.4309}\\
	{ mfeat-karhunen}&0.1777&0.0348&0.1100&0.5131&0.5796&{\bf 0.5992}\\
	{ mfeat-zerike}&0.3711&0.0430&0.0879&0.4091&{\bf 0.5428}&{0.4751}\\
    { segment}&0.0439&0.0443&0.0343&0.4858&0.5652&{\bf 0.6104}\\
	{ wavefore-5000}&0.0093&0.0068&0.0045&0.0999&{\bf 0.2827}&{0.2090}\\
	{ optdigits}&0.0358&0.0319&0.0910&0.5300&{\bf 0.7741}&{ 0.6806}\\	
	{ letter}&0.3786&0.1477&0.0382&0.3625&--&{\bf 0.4038}\\	
	{ avilia}&0.0181&0.0172&0.0139&0.0998&--&{\bf 0.1534}\\
	
	\hline
	{\bf Means}&0.2402&0.1758&0.0622&0.3137&--&{\bf 0.3957}\\
	\hline
	\end{tabular}
	\end{center}
\end{table*}

To verify that the distinctions of clustering quality (Rand Index and MNI) between the proposed algorithm and four other benchmark algorithms (the RSC algorithm is abandoned because of the null values) are statistically significant, we conduct the pairwise $t$-test \cite{Student 1908}, in which $t$-value is a statistic showing the overall difference between paired samples, $p$-value refers to the probability of null hypothesis, the smaller the $p$-value, the more significant the difference in clustering quality. Initially, we define the null hypothesis and the alternative hypothesis as $H_0: R-R^*=0$ and $H_a: R-R^* \neq 0$, in which $R$ and $R^*$ stand for Rand Index (or NMI) over the proposed algorithm and one of the benchmarks, respectively. From Table VI, we can safely conclude that, as PERCH has the highest $p$-value $2.5336 \times 10^{-2}$ in Rand Index and HAC-A has the highest $p$-value $1.903 \times 10^{-2}$ in MNI, the proposed algorithm performs better than the three benchmarks at 5\%  statistical significance level.

\begin{table}
\caption{ The results of pairwise $t$-test on fifteen UCI data sets. }
\footnotesize
\begin{center}
\begin{tabular}{|c|cc|cc|}
\hline
\multirow{2}{*}{}&\multicolumn{2}{|c|}{\textbf{Rand Index}}&\multicolumn{2}{|c|}{\textbf{MNI}}\\
\cline{2-5}
&{$t$-value}&{$p$-value}&{$t$-value}&{$p$-value}\\
\hline
{\bf HAC-A}&3.604&$3.207^{\times 10^{-3}}$&2.649&\bm{$1.903^{\times 10^{-2}}$}\\
{\bf CURE}&4.289&$8.809^{\times 10^{-4}}$&3.516&$3.420^{\times 10^{-3}}$\\
{\bf Chameleon}&4.950&$2.651^{\times 10^{-4}}$&5.719&$5.303^{\times 10^{-5}}$\\
{\bf PERCH}&2.525&\bm{$2.533^{\times 10^{-2}}$}&3.670&$2.518^{\times 10^{-3}}$\\
%
%
\hline
\end{tabular}
\end{center}
\end{table}

\subsubsection{Robustness against Input Sequence} The input sequence affects the Boundary Detection process, resulting in a fluctuation in clustering quality. Therefore, we analyze the robustness of SRSC against the input sequence by comparing its clustering qualities (Rand Index and MNI) with PERCH (which is sensitive to the input sequence) on variances. As shown in Table VII, the variances of SRSC are close to PERCH in Rand Index, meanwhile, SRSC gets significant advantages in NMI.

\begin{table*}
	\caption{Comparison of the standard deviation of the accuracy (Rand Index and NMI) of PERCH and Election Tree.}
	\begin{center}
	\footnotesize
	\begin{tabular}{|c|cc|cc|}
	\hline
	\multirow{2}{*}{\bf Datasets}&\multicolumn{2}{|c|}{\bf Rand Index}&\multicolumn{2}{|c|}{\bf NMI}\\
	\cline{2-5}
    \multirow{2}{*}{}&{\bf PERCH}&{\bf SRSC}&{\bf PERCH}&{\bf SRSC}\\
	\hline
    iris&$6.71\times 10^{-2}$&\bm{$1.41\times 10^{-2}$}&$1.33\times 10^{-1}$&\bm{$6.78\times 10^{-4}$}\\	
    {sonar}&$1.51\times 10^{-2}$&\bm{$6.70\times 10^{-3}$}&$2.25\times 10^{-2}$&\bm{$9.95\times 10^{-3}$}\\
	{glass}&\bm{$1.99\times 10^{-2}$}&$3.18\times 10^{-2}$&$5.29\times 10^{-2}$&\bm{$1.96\times 10^{-3}$}\\	
    {ecoli}&$2.19\times 10^{-2}$&\bm{$1.12\times 10^{-3}$}&$5.61\times 10^{-2}$& \bm{$2.92\times 10^{-2}$}\\
	{ ionosphere}&$2.90\times 10^{-2}$&\bm{$1.50\times 10^{-4}$}&$4.57\times 10^{-2}$&\bm{$1.59\times 10^{-3}$}\\
	{ synthentic control}&{$2.41\times 10^{-2}$}&\bm{$1.69\times 10^{-2}$}&$7.01\times 10^{-2}$&\bm{$2.38\times 10^{-2}$}\\
    { vehivel}&\bm{$1.71\times 10^{-2}$}&$3.31\times 10^{-2}$&$3.55\times 10^{-2}$&\bm{$1.03\times 10^{-2}$}\\
	{ mfeat-fourier}&\bm{$7.36\times 10^{-3}$}&$1.08\times 10^{-2}$&$3.54\times 10^{-2}$&\bm{$1.26\times 10^{-2}$}\\
	{ mfeat-karhunen}&\bm{$9.98\times 10^{-3}$}&$2.22\times 10^{-2}$&$4.49\times 10^{-2}$&\bm{$4.00\times 10^{-2}$}\\
	{ mfeat-zerike}&\bm{$7.40\times 10^{-3}$}&$1.26\times 10^{-2}$&$3.51\times 10^{-2}$&\bm{$1.99\times 10^{-2}$}\\
    { segment}&$2.00\times 10^{-2}$&\bf{0.00}&$5.82\times 10^{-2}$&\bf{0.00}\\
	{ wavefore-5000}&\bm{$1.92\times 10^{-2}$}&$2.79\times 10^{-2}$&$4.83\times 10^{-2}$&\bm{$3.32\times 10^{-2}$}\\
	{ optdigits}&$1.47\times 10^{-2}$&\bm{$1.08\times 10^{-2}$}&$4.19\times 10^{-2}$&\bm{$2.33\times 10^{-2}$}\\	
	{ letter}&\bm{$9.55\times 10^{-3}$}&$1.40\times 10^{-2}$&\bm{$1.11\times 10^{-2}$}&{ $1.16\times 10^{-2}$}\\	
	{ avilia}&\bm{$4.28\times 10^{-3}$}&$1.37\times 10^{-2}$&\bm{$1.61\times 10^{-2}$}&$1.74\times 10^{-2}$\\
	\hline
	{\bf Means}&$1.91\times 10^{-2}$&\bm{$1.44\times 10^{-2}$}&$4.71\times 10^{-2}$&\bm{$1.40\times 10^{-2}$}\\
	\hline
	\end{tabular}
	\end{center}
\end{table*}

\section{Complexity Analysis and CPU Time}

Benefiting from the representatives of clusters (roots), the consumption of calculation is decreasing along with the level of cluster tree, since the amount of roots will be reduced by at least half for each iteration. 

There are two key steps in each iteration for the proposed algorithm: (i) construction of sub-MSTs and (ii) roots detection. In the first step, data points link together depending on the searching algorithm and the type of dataset. For instance, in network data sets, the average degree is the main factor that affects the time-consumption of nearest neighbor searching, while in text or image related data sets, the nearest neighbor searching operation should be performed on a fully connected graph which is supposed to incredibly increase the time complexity. Generally, for multi-dimensional data, algorithms such as K-D Tree are preferred to find data points' nearest neighbors \cite{Bentley 1975}, of which the time-complexity is $O(n\log n)$. Likewise, the construction of sub-MSTs mainly relies on the nearest neighbors search. Hence, the time-complexity of the first step in one iteration is also $O(n\log n)$. Considering the worst case, given a dataset $X=\{x_i\}^n$, each data point owns the reciprocal nearest neighbor data point. Also, every root that is detected in each iteration has its reciprocal nearest neighbor root. Similar to the analysis made in our previous work \cite{Xie 2020}, the number of the roots is exactly the half size of data points in each iteration, i.e., we will get $\frac{n}{2}$ of roots in the first iteration, and $\frac{n}{4}$ of roots in the second iteration, and so forth. Therefore, the time-complexity of the first step could be calculated as,
\begin{eqnarray}
	T_1 &=&\left(\underbrace{n\log n}_{iteration\#1}\right)+\left(\underbrace{\frac{n}{2^1}\log\frac{n}{2^1}}_{iteration\#2}\right)+\left(\underbrace{\frac{n}{2^2}\log\frac{n}{2^2}}_{iteration\#3}\right)+\cdots \nonumber\\
	&<&\log n\sum_{i=0}\frac{n}{2^i} \nonumber\\
	&<&2n\log n .\nonumber
\end{eqnarray}

In the second step, the time-complexity is based on the calculation of indices that depends on the BFS (breadth-first searching), of which the time-complexity is $O(n)$ in one iteration. In addition, when the scoring method fails, the boundary-sampling process will be triggered, and every root will be compared to the boundary point pairs. Also considering the worst case we mentioned above, the time-complexity of the second step could be calculated as, 
\begin{eqnarray}
	T_2 &=& \left(\underbrace{n\!+\!n\log n}_{iteration\#1}\right)\!+\!\left(\underbrace{\frac{n}{2^1}\!+\!\frac{n}{2^1}\log \frac{n}{2^1}}_{iteration\#2}\right)\!+\!\left(\underbrace{\frac{n}{2^2}+\frac{n}{2^2}\log \frac{n}{2^2}}_{iteration\#3}\right)\!+\!\cdots \nonumber\\
	&=& \sum_{i=0} \left(\frac{n}{2^i}\!+\!\frac{n}{2^i}\log\frac{n}{2^i}\right)
	\nonumber\\
	&<& \sum_{i=0}\frac{n}{2^i} +\log n \sum_{i=0}\frac{n}{2^i} \nonumber\\
	&<& 2n+2n\log n .\nonumber
\end{eqnarray}

Finally, the overall time consumption of the proposed algorithm $T=T_1+T_2<2n+4n\log n$, thus the time-complexity is $O(n\log n)$, which is much better than many classical hierarchical clustering algorithms. Meanwhile, in the worst case, a maximum of $\log n$ boundary point pairs are required to be resided in memory for scoring anomalous RNNs nodes, that is, the space-complexity of the proposed algorithm is $O(\log n)$.

We then verify the efficiency of SRSC by checking its required CPU Time on different sizes of random data sets. As shown in Fig. 9, we compare the CPU Time of SRSC with other baselines on different sizes of artificial data sets that are generated randomly, in which, the pow-law is used to fit the relation between CPU times and data sizes for each algorithm. The exponents emphasized with corresponding colors in Fig. 9 show that the proposed algorithm is of better efficiency than other benchmarks.

\begin{figure}
\label{fig9}
\centering
  \includegraphics[width=0.7\linewidth]{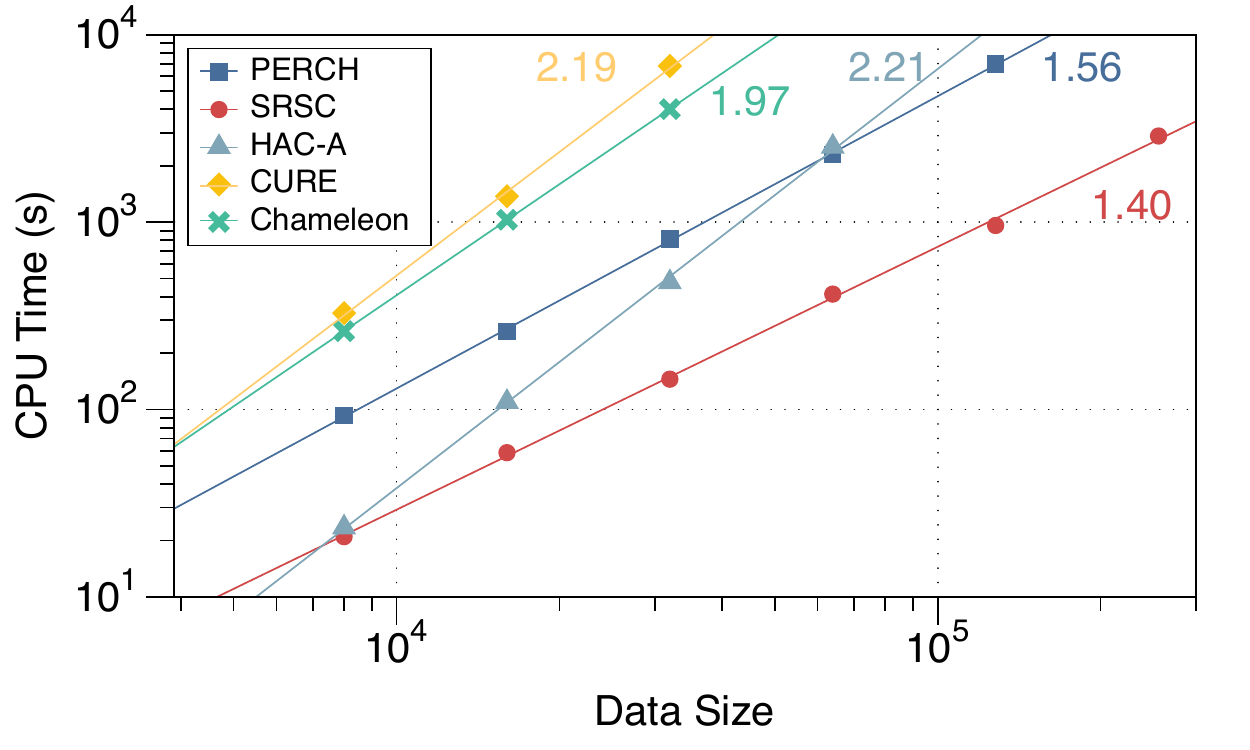}  
\caption{{\bf CPU Times on random data sets.} }
\end{figure}

\section{Conclusion}

In this article, we mainly worked on the hierarchical clustering, an important branch of clustering algorithms. A key issue for the hierarchical clustering models is to decide the root for a cluster properly and reasonably, which is also left over from our previous work \cite{Xie 2020}, as a better selected root or the representative data point of the detected cluster is going to potentially impact the final results of the clustering. In this study, we proposed four independent indices and two hybrid measures to score the reciprocal nearest neighbor data points with a graph-based theory. In addition, we further developed a theoretical framework to successfully handle the cases of anomalous scoring on RNNs. Extensive experiments on 15 datasets verified that the proposed algorithm significantly outperform the benchmark algorithms, including 3 classical hierarchical algorithms and two state-of-the-art methods PERCH and RSC. Moreover, the model analyses concluded that the proposed algorithm has the time complexity by $O(n\log n)$ and the space complexity by $O(\log n)$, both of which are much lower than the traditional hierarchical clustering models.

Despite the encouraging results we have achieved in this study, we also noticed some interesting directions which could be regarded as the extension of the future work. The strategy we devised for detecting the root is effective yet hard to say it is perfect already. Some new explorations are possible to conduct as we have seen some advanced node scoring methods recently appeared in the study of complex network analysis \cite{Lu 2016, Lee 2017}. On the other hand, currently, the role of root is not clear in the procedure of hierarchical clustering, namely why it can bond the lower-level clustering and upper-level clustering in a dendrogram effectively. Some further studies on this regard are imminent for us to carry out which would broaden our understandings on this issue.

\section*{Acknowledgment}
\noindent {\bf Funding:} 
 This work is partially supported by the National Natural Science Foundation of China under Grant No. 61703074, and the Fundamental Research Funds for the Central Universities under Grant No. ZYGX2016J196, and the Young Scholars Development Fund of SWPU under Grant No. 20219\\9010142.


\section*{References}
\begin{thebibliography}{00}


\bibitem{Jain 2010} A. K. Jain, ``Data clustering: 50 years beyond k-means'', {\it Pattern Recognit. Lett.}, vol. 31, pp. 651-666, 2010.

\bibitem{Shah 2017}S. A. Shah and V. Koltun, ``Robust continuous clustering'', {\it Proc. Natl. Acad. Sci. U.S.A.}, vol. 114, pp. 9814-9819, 2017.

\bibitem{Rodrigue 2014} A. Rodrigue and A. Laio, ``Clustering by fast search and find of density peaks'', {\it Science}, vol. 344, pp. 1492-1496, 2014.

\bibitem{Wang 2009} X. F. Wang and D.S. Huang, ``A novel density-based clustering framework by using level set method'', {\it IEEE Trans. Knowl. Data Eng.}, vol. 21, pp. 1515-1531, 2009.

\bibitem{Frey 2007} B. J. Frey and D. Dueck, ``Clustering by passing messages between data points'', {\it Science}, vol. 315, pp. 972-976, 2007.

\bibitem{Filippone 2008} M. Filippone, F. Camastra, F. Masulli and S. Rovetta, ``A survey of kernel and spectral methods for clustering'', {\it Pattern Recogn.}, vol. 41, pp. 176-190, 2008.

\bibitem{Huang 2019} D. Huang, C. D. Wang,  J.S. Wu, J. H. Lai  and C. K. Kwoh,  ``Ultra-scalable spectral clustering and ensemble clustering'', {\it IEEE Trans. Knowl. Data Eng.}, vol. 32, pp. 1212-1226, 2019.

\bibitem{Zhang 1997} T. Zhang, R. Ramakrishnan and M. Livny, ``Birch: A new data clustering algorithm and its applications'', {\it Data Min. Knowl. Discov.}, vol. 1, pp. 141-182, 1997.

\bibitem{Kobren 2017} A. Kobren, N. Monath, A. Krishnamurthy and A. McCallum, ``A hierarchical algorithm for extreme clustering'', in {\it 23rd ACM SIGKDD Conf. Knowl. Discovery Data Mining}, 2017, pp. 255-264.

\bibitem{Ryu 2020} H. C. Ryu, S. Jung and S. Pramanik, 2019. ``An effective clustering method over CF+ tree using multiple range queries'', {\it IEEE Trans. Knowl. Data Eng.}, vol. 32, pp. 1694-1706, 2020.

\bibitem{Song 2016} L. Song, et al., ``A transcription factor hierarchy defines an environmental stress response network'', {\it Science}, vol. 354, pp. aag1550, 2016.

\bibitem{Ma 2017} D. L. B. Ma, et al., ``Predicting gene regulatory networks by combining spatial and temporal gene expression data in Arabidopsis root stem cells'', {\it Proc. Natl. Acad. Sci. U.S.A.}, vol. 114, pp. E7632-E7640, 2017.

\bibitem{Maganga 2020} G. D. Maganga, et al. ``Genetic diversity and ecology of coronaviruses hosted by cave-dwelling bats in Gabon'' {\it Sci. Rep.}, vol. 10, no. 7314, 2020.


\bibitem{Newman 2004} M. E. J. Newman, ``Fast algorithm for detecting community structure in networks'', {\it Phys. Rev. E}, vol. 69, pp. 066133, 2004.

\bibitem{Wilkinson 2004} D. M. Wilkinson and B. A. Huberman, ``A method for finding communities of related genes'', {\it Proc. Natl. Acad. Sci. U.S.A.}, vol. 101, pp. 5241-5248, 2004.
\bibitem{Rattigan 2007} M. J. Rattigan, M. Maier and D. Jensen, ``Graph clustering with network structure indices'', in {\it Proc. 24th Int. Conf. Mach. Learn.}, 2007, pp. 783-790.
\bibitem{Fortunato 2010} S. Fortunato, ``Community detection in graphs'', {\it Phys. Rep.}, vol. 486, pp. 75-174, 2010.


\bibitem{Dugan 2017} H. A. Dugan, et al., ``Salting our freshwater lakes'', {\it Proc. Natl. Acad. Sci. U.S.A.}, vol. 114, pp. 4453-4458, 2017.

\bibitem{Reddy 2018} C. K. Reddy and B. Vinzamuri, ``A Survey of Partitional and Hierarchical Clustering Algorithms'', in {\it Data Clustering: Algorithms and Applications}, Ed. New York, Chapman and Hall/CRC, 2018, ch. 4, pp. 88-107.
\bibitem{Ward Jr 1963} J. H. Ward Jr,  ``Hierarchical grouping to optimize an objective function'', {\it J. Amer. Statist. Assoc.}, vol. 58, pp. 236-244, 1963.

\bibitem{Sneath 1973} P. H. A. Sneath and R. R. Sokal, {\it Numerical taxonomy. The principles and practice of numerical classification}, Ed. San Francisco, W.H. Freeman Co., 1973.

\bibitem{King 1967} B. King, ``Step-wise Clustering Procedures'', {\it J. Am. Stat. Assoc.}, vol. 69, pp. 86-101, 1967.

\bibitem{Rokach 2009} L. Rokach, ``A survey of Clustering Algorithms'', in {\it Data Mining and Knowledge Discovery Handbook}, O. Maimon and L. Rokach Eds. Boston, Springer, 2010, pp. 269-298.


\bibitem{Guha 2000} S. Guha, R. Rastogi and K. Shim, ``Rock: a robust clustering algorithm for categorical attributes'', {\it Inf. Syst.}, vol. 25, pp. 345-366, 2000.
\bibitem{Guha 2001} S. Guha, R. Rastogi and K. Shim, ``Cure: an efficient clustering algorithm for large database'', {\it Inf. Syst.}, vol. 26, pp. 35-58, 2001.
\bibitem{Karypis 1999} G. Karypis, E. H. Han and V. Kumar, ``Chameleon: hierarchical clustering using dynamic modeling'', {\it Computer}, vol. 32, pp. 68-75, 1999.
\bibitem{Xie 2020} W. B. Xie, Y. L. Lee, C. Wang, D. B. Chen and T. Zhou, ``Hierarchical clustering supported by reciprocal nearest neighbors'', {\it Info. Sci.}, vol. 527, pp.  279-292, 2020.

\bibitem{Murtagh 1983} F. Murtagh, ``A survey of recent advances in hierarchical clustering algorithms'', {\it The Comput. J.}, vol. 26, pp. 354-359, 1983.

 \bibitem{Lopez-Sastre 2012} R. J. L\`{o}pez-Sastre, D. O\~{n}oro-Rubio, P. Gil-Jim\'{e}nez and S. Maldonado-Basc\'{o}n, ``Fast reciprocal nearest neighbors clustering'', {\it Signal Proc.}, vol. 92, pp. 270-275, 2012.
 
\bibitem{Birant 2007} D. Birant and A. Kut, ``ST-DBSCAN: An algorithm for clustering spatial-temporal data'', {\it Data Knowl. Eng.}, vol. 60, pp. 208-221, 2007.

\bibitem{Krishnamurthy 2012} A. Krishnamurthy, S. Balakrishnan, M. Xu and A. Singh, ``Efficient active algorithms for hierarchical clustering'', in {\it Int. Conf. Mach. Learn.}, 2012, pp. 267-274.

\bibitem{Bouguettaya 2015} A. Bouguettaya, Q. Yu, X. M. Liu, X. M. Zhou and A. Song, ``Efficient agglomerative hierarchical clustering'', {\it Expert Syst. Appl.}, vol. 42, no. 5, pp. 2785-2797, 2015.

\bibitem{Lichman 2013} M. Lichman, 2013, ``UCI Machine Learning Repository'', Irvine, CA: University of California, School of Information and Computer Science. [Online]. Available: http://archive.ics.uci.edu/ml

\bibitem{Rand 1971} W. M. Rand, ``Objective criteria for the evaluation of clustering methods'', {\it J. Amer. Statist. Assoc.}, vol. 66, pp. 846-850, 1971.

\bibitem{Cover 2005} T. M. Cover, J. A. Thomas, ``Entropy, Relative Entropy, and Mutual Information'', in {\it Elements of Information Theory},  T. M. Cover, J. A. Thomas, Eds., John Wiley \& Sons, Ltd, 2005. 

\bibitem{Student 1908} Student, ``The Probable Error of a Mean'', in {\it Breakthroughs in Statistics}, S. Kotz S. and N. L. Johnson, Eds., New York, Springer, 1992.

\bibitem{Bentley 1975} J. L. Bentley, ``Multidimensional binary search trees used for associative searching'', {\it Commun. ACM}, vol. 18, pp. 509-517, 1975.

\bibitem{Lu 2016} L. Lu, D. Chen, X. Ren,Q. Zhang, Y. Zhang and T. Zhou, ``Vital nodes identification in complex networks'', {\it Phys. Rep.}, vol. 650, pp. 1-63, 2016.

\bibitem{Lee 2017} Y. L. Lee and T. Zhou, ``Fast asynchronous updating algorithms for k-shell indices'', {\it Phys. A}, vol. 482, pp. 524-531, 2017.


\end{thebibliography}
\end{document}